\documentclass[final,12pt]{clear2024} 

\title[On Counterfactual Data Augmentation Under Confounding]{On Counterfactual Data Augmentation Under Confounding}
\usepackage{times}
\usepackage{tikz}
\usepackage{microtype}
\usepackage{float}

\usepackage{algorithm}
\usepackage{algpseudocode}
\usepackage{booktabs} 
\usepackage[export]{adjustbox}
\usepackage{times}  
\usepackage{helvet}  
\usepackage{courier}  
\newcommand{\indep}{\perp \!\!\! \perp}

\SetCommentSty{mycommfont}
\usepackage{caption} 
\usepackage{subcaption}
\DeclareCaptionStyle{ruled}{labelfont=normalfont,labelsep=colon,strut=off} 
\usepackage{multirow}
\usepackage{cancel}
\usepackage{enumitem}
\usepackage{mathtools}
\usepackage{thmtools}

\DeclareMathOperator*{\argmin}{arg\,min}
\usepackage{xcolor}

\usepackage{wrapfig}
\usepackage{listings}
\usepackage{xcolor}
\usepackage{tikz}
\usetikzlibrary{cd}
\newcommand*\circled[1]{\tikz[baseline=(char.base)]{
            \node[shape=circle,draw,inner sep=2pt] (char) {#1};}}

\definecolor{codegreen}{rgb}{0,0.6,0}
\definecolor{codegray}{rgb}{0.5,0.5,0.5}
\definecolor{codepurple}{rgb}{0.58,0,0.82}
\definecolor{backcolour}{rgb}{0.95,0.95,0.92}

\lstdefinestyle{codestyle}{
    backgroundcolor=\color{backcolour},   
    commentstyle=\color{codegreen},
    keywordstyle=\color{magenta},
    numberstyle=\tiny\color{codegray},
    stringstyle=\color{codepurple},
    basicstyle=\ttfamily\footnotesize,
    breakatwhitespace=false,         
    breaklines=true,                 
    captionpos=b,                    
    keepspaces=true,                 
    numbers=left,                    
    numbersep=5pt,                  
    showspaces=false,                
    showstringspaces=false,
    showtabs=false,                  
    tabsize=2
}
\usepackage[utf8]{inputenc} 
\usepackage[T1]{fontenc}    
\usepackage{url}            
\usepackage{booktabs}       
\usepackage{amsfonts}       
\usepackage{nicefrac}       
\usepackage{microtype}      
\usepackage{xcolor}         

\clearauthor{%
 \Name{Abbavaram Gowtham Reddy} \Email{cs19resch11002@iith.ac.in}\\
 \addr Indian Institute of Technology Hyderabad
 \AND
 \Name{Saketh Bachu} \Email{saketh.bachu@cse.iith.ac.in}\\
 \addr Indian Institute of Technology Hyderabad
 \AND
 \Name{Saloni Dash} \Email{t-sadash@microsoft.com}\\
 \addr Microsoft Research
 \AND
 \Name{Charchit Sharma} \Email{charchit.sharma@cse.iith.ac.in}\\
 \addr Indian Institute of Technology Hyderabad
 \AND
 \Name{Amit Sharma} \Email{amshar@microsoft.com}\\
 \addr Microsoft Research
 \AND
 \Name{Vineeth N Balasubramanian} \Email{vineethnb@iith.ac.in}\\
 \addr Indian Institute of Technology Hyderabad
}

\begin{document}

\maketitle

\begin{abstract}%
Counterfactual data augmentation has recently emerged as a method to mitigate confounding biases in the training data. These biases, such as spurious correlations, arise due to various observed and unobserved confounding variables in the data generation process. In this paper, we formally analyze how confounding biases impact downstream classifiers and present a causal viewpoint to the solutions based on counterfactual data augmentation. We explore how removing confounding biases serves as a means to learn invariant features, ultimately aiding in generalization beyond the observed data distribution. Additionally, we present a straightforward yet powerful algorithm for generating counterfactual images, which effectively mitigates the influence of confounding effects on downstream classifiers. Through experiments on MNIST variants and the CelebA datasets, we demonstrate how our simple augmentation method helps existing state-of-the-art methods achieve good results.
\end{abstract}

\begin{keywords}%
  Counterfactuals, Augmentation, Confounding, Bias, Correlation, Causality.
\end{keywords}

\section{Introduction}
A confounding variable is one that \textit{causes} two (or more) other variables, potentially creating spurious correlations between them. The presence of confounders is a challenge when working with real-world data, as the consequent spurious correlations make it difficult to identify reliable features that accurately represent the target label in machine learning applications~\citep{anchor,maximin,oodgen}.
For instance, the \textit{geographical location} where an individual resides can potentially cause both their \textit{race} and the level of \textit{education} they receive. When using such observational data to train a machine learning model that predicts an individual's \textit{income}, the model may inadvertently exploit the spurious correlations between \textit{race} and \textit{education}, leading to unfair \textit{income} predictions for individuals of different \textit{racial} backgrounds. Addressing confounding biases in trained machine learning models has demonstrated its usefulness in various applications such as zero or few-shot learning~\citep{causalviewofcompo,Yue_2021_CVPR}, disentanglement~\citep{suter2018robustly,reddy2022causally}, domain generalization~\citep{cgn,ali,ilse}, algorithmic fairness~\citep{kilbertus2020sensitivity,kilbertus2020fair} and healthcare~\citep{modelpatching,zhao2020training}. However, very few efforts have explicitly studied confounding bias in the context of data augmentation techniques.

Confounding in observational data poses substantial challenges for learning models, regardless of whether the confounding variables are observed or unobserved: (i) when confounders are present, disentanglement of features exhibiting spurious correlations through generative modeling becomes an arduous task~\citep{cgn,reddy2022causally,funke2022disentanglement}; (ii) it is infeasible to identify underlying generative factors without additional supervision~\citep{von2021self,causal_repr_learning}; and (iii) in the presence of confounders, classifiers may rely on non-causal features to make predictions~\citep{causal_repr_learning}. Recent endeavors have studied and attempted to address spurious correlations stemming from confounding effects in observational data~\citep{trauble,cgn,modelpatching,ilse,oodgen,provably,arjovsky2019invariant}. In this work, we study a lesser studied topic in this context -- the efficacy of counterfactual data augmentation for mitigating confounding in deep neural network (DNN) models, with a focus on image data. 

Many methods have been proposed for data augmentation in general to improve the performance of DNN models~\citep{Shorten2019}. Fewer efforts have studied this from a causal perspective; these studies have focused on issues such as interventions~\citep{ilse}, out-of-distribution generalization~\citep{oodgen}, model patching~\citep{modelpatching} or generative models~\citep{cgn}. The proposed work presents a different perspective by introducing a novel causal perspective on data augmentation and presents a careful study on how existing data augmentation techniques enable specific interventional queries within the underlying causal graph, leading to the generation of augmented data.

\begin{figure}
    \centering
    \scalebox{0.57}{
    \tikzset{every picture/.style={line width=0.75pt}} 
    \input{images/intro.tikz}
    }
    \vspace{-6pt}
    \caption{\footnotesize We illustratively show why it is useful to study a causal perspective to choose an appropriate intervention for mitigating confounding bias in data augmentations. (a) True causal graph $\mathcal{G}$ and inference procedure that utilizes the learned representation $\phi(X)$ of $X$ to predict the label $\hat{Y}$. $Z_0,Z_1,\dots,Z_n$ are generative factors, $U_1,\dots,U_n$ are confounding variables that may create spurious correlations among generative factors, and $Y$ is the true label. Gray-colored nodes represent observed variables. In the case of the double-colored MNIST dataset discussed herein, $Z_0$ is the causal feature (\textit{shape} of a digit) and $Z_1,\dots,Z_n$ capture other generative factors (e.g., \textit{background color, foreground color}) to form a real-world image $X$. (b) Causal model (defined by the structural equations) based on same graph $\mathcal{G}$ and corresponding samples from the double-colored MNIST dataset distribution generated from that causal model. Note that images in (b) encode confounding bias; for e.g., digit $1$ most often has a white foreground and green background. (c) Causal graph $\mathcal{G}_{do(X)}$ is an intervened causal graph derived from $\mathcal{G}$ by removing all incoming arrows to $X$, thus removing any backdoor paths from the confounders $U_i$s to $\hat{Y}$. We implement this using a CutMix~\citep{cutmix} augmentation derived from putting together randomly extracted image patches from other images. Note that this does not explicitly remove confounding bias in the generated images. (d) Causal graph $\mathcal{G}_{do(Z_0)}$ is an intervened causal graph derived from $\mathcal{G}$ by removing all incoming arrows to $Z_0$. Such an intervention helps remove the confounding bias in this case.}
    \vspace{-14pt}
    \label{fig:datagen}
\end{figure}
To comprehend the importance of a causal interpretation of data augmentation, consider the causal graph $\mathcal{G}$ from Figure~\ref{fig:datagen} (a) that captures many real-world causal generative processes~\citep{suter2018robustly,von2021self,ilse,reddy2022causally}. In $\mathcal{G}$, the causal feature $Z_0$ (e.g., \textit{shape} of a digit) and a set of generative factors $Z_1,\dots,Z_n$ (e.g., \textit{background color, foreground color}) form a real-world image $X$ (e.g., an image of handwritten digit \textit{1} with \textit{white foreground color and green background color as shown in   Figure~\ref{fig:datagen}} (b)) through an unknown causal mechanism $g$ i.e., $X=g(Z_0,Z_1,\dots,Z_n)$. Each $Z_i; i\in \{0,\dots,n\}$ is a function of exogenous noise variables $U_1,\dots,U_m$ that serve as confounders between pairs of generative factors $Z_0,\dots,Z_n$. Specifically, $Z_i=f_i(pa_{Z_i}); i\in \{0,\dots,n\}$ where $f_i$ is the causal mechanism for generating $Z_i$ and $pa_{Z_i}\subseteq \{U_1,\dots,U_m\}$ is the set of parents of $Z_i$. $Z_0, \dots, Z_n$ are confounded by $U_1,\dots,U_m$ that may be observed or unobserved (e.g., certain digits appear only in a certain combination of foreground and background colors). We note that this is an illustrative example, and our analysis remains valid even when the number of causal features exceeds one and even when not all exogenous noise variables cause all of the variables $Z_0, \ldots, Z_n$. Due to the presence of confounding variables $U_1,\dots,U_m$, models trained on $X$ may face challenges in predicting the true label $Y$ because in addition to a causal path $Z_0\rightarrow X\rightarrow \phi(X)\rightarrow \hat{Y}$ to the predicted label $\hat{Y}$, the causal feature $Z_0$ has \textit{back-door} paths~\citep{pearl2009causality} $Z_0\leftarrow U_j\rightarrow Z_i\rightarrow X\rightarrow \phi(X)\rightarrow \hat{Y}$ to $\hat{Y}$ for some $j\in\{1,\dots,m\}, i\in \{1,\dots,n\}$ that induce spurious correlations between causal feature $Z_0$ and non-causal features $Z_i; i\neq 0$. (We provide a concise overview of fundamental concepts essential for understanding our paper in Appendix \S~\ref{sec: preliminaries}.)

Traditional counterfactual data augmentation methods aim to augment the original data $\mathcal{D}$ with new data $\mathcal{D}'$ in order to create the augmented dataset $\mathcal{D}_{aug} = \mathcal{D} \cup \mathcal{D}'$. $\mathcal{D}_{aug}$ is often intended to capture an \textit{intervened} causal graph $\mathcal{G}_{do(\cdot)}$ in which there are no back-door paths from the confounders to $X$; however, not all data augmentation techniques can block back-door paths to effectively remove confounding effects (see Figure~\ref{fig:datagen} (c) and (d)). 
For instance, in the intervened causal graph $\mathcal{G}_{do(X)}$ of Figure~\ref{fig:datagen} (c), although there are no backdoor paths from the confounding variables to $X$, the confounding implicit in $X$ cannot be eliminated (i.e., in any patch of newly generated images, the combination of \textit{digit shape, foreground, background colors} remains unchanged). Also, the causal path $Z_0\rightarrow X$ has been removed in $\mathcal{G}_{do(X)}$, making it challenging to learn causal features from $X$.  It is worth noting that not all data augmentation techniques are universally applicable in all applications. For instance, as demonstrated in Figure~\ref{fig:datagen} (d), performing an intervention $do(Z_i=z_i)$ for $i\neq 0$ may be non-trivial. 
Given this background, in this paper, we adopt a causal perspective to investigate data augmentations and offer insights into existing methods that address confounding effects in observational data. \textit{Our objective herein is not to outperform state-of-the-art accuracy scores; rather, we aim to present a new causal perspective, and thereby, correct and simple procedures, for performing data augmentation when confronted with data that exhibit confounding effects and their corresponding utility on well-known tasks.} The main contributions of this paper can be summarized as follows.
 \vspace{-8pt}
\begin{itemize}[leftmargin=*]
\setlength \itemsep{-0.1em}
    \item We introduce a formal framework for quantifying the extent of confounding and investigate its relation with the non-linear dependency between pairs of generative factors (\S~\ref{sec: info theoretic}).
    \item We analyze the efficacy of counterfactual data augmentation in mitigating confounding bias, leveraging intervened causal model as a key tool (\S~\ref{sec: removing confounding effects}).
    \item We demonstrate the impact of confounding removal on achieving out-of-distribution generalization and learning invariant features (\S~\ref{sec: invariant}). We then propose a straightforward algorithm that enables the generation of counterfactual data, effectively eliminating confounding bias (\S~\ref{sec: algorithm}).
    \item Through extensive experiments conducted on widely recognized benchmarks, including three variants of the MNIST dataset and the CelebA dataset, we evaluate the effectiveness of our augmentation approach in conjunction with different methods and their utility on the performance of a downstream classifier against other augmentation methods (\S~\ref{sec: experiments and results}).
\end{itemize}

\section{Related Work}

\label{relatedwork}

\noindent \textbf{Image Data Augmentation:} Image data augmentation plays a crucial role in enhancing the performance and robustness of deep learning models in computer vision tasks. Numerous studies have extensively explored diverse techniques and strategies for augmenting image data. These efforts aim to achieve several objectives, including increasing the diversity of datasets, mitigating overfitting, improving generalization capabilities~\citep{imagenet,simonyan2014very,yang2022image}, strengthening resilience against adversarial attacks~\citep{madry2018towards,xie2020unsupervised}, facilitating domain generalization~\citep{ilse}, promoting algorithmic fairness~\citep{dataaugfair}, and more. Image data augmentations encompass a wide range of approaches, ranging from traditional image manipulation techniques such as rotation, flipping, cropping, among others~\citep{imagenet,simonyan2014very,perez2017effectiveness,augmix,cutout, mixup, cutmix,ilse}, to more recent generative-based augmentations~\citep{antoniou2017data,cgn,oodgen,modelpatching} that manipulate higher-level semantic aspects of an image, such as \textit{smiling} or \textit{hair color}.

\noindent \textbf{Counterfactual Data Augmentation:} Conventional data augmentation techniques, including rotation, scaling, and corruption, lack the ability to modify the underlying causal generative process. Consequently, they are unable to effectively mitigate confounding biases. For instance, rotation and scaling cannot \textit{separate} the color and shape of an object in an image. To overcome this limitation, counterfactual data augmentation has emerged as a promising approach~\citep{cgn,oodgen,modelpatching, kusner2017counterfactual,pitis2020counterfactual,denton2019detecting}. Counterfactual inference enables fine-grained control over the generative factors, allowing for the generation of new samples that effectively address confounding biases.

Pearl's influential contribution to the field of causality~\citep{pearl2009causality} presents a three-step methodology for generating counterfactual instances, encompassing the identification of underlying generative factors and the structural causal model (SCM). Recent research endeavors have focused on modeling the SCM under different assumptions, facilitating the generation of counterfactual images through targeted interventions within the learned model. The efficacy of counterfactual data augmentation has been substantiated across diverse real-world domains, encompassing applications such as fair classification~\citep{kusner2017counterfactual,denton2019detecting}, causal explanations~\citep{zmigrod2019counterfactual, pitis2020counterfactual,counterfactual_continuous,tractableinference}, identification of biases in real-world applications~\citep{joo2020gender}, and counterfactual data augmentation for reinforcement learning~\citep{pitis2020counterfactual}.   

A recent method known as Counterfactual Generative Networks (CGN)~\citep{cgn} assumes that each image is a result of a composition of three fixed generative factors: \textit{shape, texture, and background}. CGN trains a generative model that learns separate independent causal mechanisms for shape, texture, and background, and combines them deterministically to generate observations. By intervening on these learned mechanisms, counterfactual data can be sampled. However, the fixed architecture of CGN, which assumes a specific number and types of mechanisms (shape, texture, background), lacks generality and may not directly apply to scenarios where the number of underlying generative factors are more/unknown. Additionally, it is unnecessary to learn every causal mechanism in the underlying causal process to address a specific confounding bias in the data. Recently, CycleGANs~\citep{cycleGAN} have been utilized to generate counterfactual data points~\citep{modelpatching, oodgen}. Using CycleGANs, a transformation is learned between two image domains, and this learned transformation is employed to generate new images. These methods employ counterfactual data augmentation to address specific problems without formally analyzing the choice of data augmentation. Our study demonstrates that achieving confounding removal does not necessitate interventions on all generative factors. Instead, we propose a straightforward solution that involves intervening on a few generative factors.

Recently,~\citep{ilse} conducted a formal analysis of data augmentations from a causal perspective. In contrast to their work, we present a formal study that examines multiple approaches to data augmentation, analyzing their individual effectiveness in mitigating confounding bias through the use of a confounding measure.

\section{Preliminaries}
\label{sec: background}

Let $\mathbf{Z}=\{Z_i\}_{i=0}^n$ be a set of $n$ random variables denoting the generative factors of an observed variable $X$, and $Y$ be the observed (true) label of $X$. $Z_0$ is the causal feature such that the label $Y$ of $X$ is caused only by $Z_0$. Note that $Z_0$ can also be a set of variables that causally influence the output in general; without loss of generality, we treat it as a singleton set in this work for convenience of understanding and analysis. Variables in $\mathbf{Z}$ may potentially be confounded by a set of $m$ confounders $\mathbf{U}=\{U_1,\dots,U_m\}$ that denote real-world confounding factors such as selection bias, spurious correlations. Let $p_{\mathbf{U}}=\prod_{i=1}^m p_{U_i}$ be the joint probability distribution of $\mathbf{U}$ and $p_{Z_i}$ be the marginal probability distribution of $Z_i;\ \ \forall i\in\{0,\dots,n\}$. $\mathcal{G} = (\mathcal{V}, \mathcal{E})$ is the causal graph denoting the causal relationships among the set of variables $\mathcal{V}=\mathbf{Z}\cup \mathbf{U} \cup \{X,Y\}$. $\mathcal{E}$ is the set of directed edges among the variables in $\mathcal{V}$ denoting the directionality of causal influences. Let $pa_{Z_i} = \{U_j|U_j\rightarrow Z_i\}$ be the set of parents of $Z_i$. Each $Z_i$ can be viewed as an outcome of a causal mechanism $f_i$ with inputs $pa_{Z_i}$. $\mathcal{G}$ in Figure~\ref{fig:datagen} (a) illustrates the graphical representation of causal processes described above. Let $\mathcal{D}=\{(X_i, Y_i)\}_{i=1}^N$ be a set of $N$ input and label pairs where each observation $X_i$ is generated from the variables in $\mathbf{Z}$ through an unknown invertible causal mechanism $g$. Formally, the generative model for $X$ can be written as follows.
\vspace{-3pt}
\begin{equation}
\label{datagen}
    \mathbf{U}\sim p_{\mathbf{U}},\hspace{1.5cm} Z_i\coloneqq f_i(pa_{Z_i}), \hspace{1.5cm} X\coloneqq g(\mathbf{Z})
\end{equation}
During inference, when presented with an input $X$, it is essential to utilize the causal feature $Z_0$ of $X$ to predict $\hat{Y}$ (see Figure~\ref{fig:datagen} (a)). Nevertheless, presence of confounding variables $\mathbf{U}$ introduce non-causal or backdoor paths from $Z_0$ to $\hat{Y}$ through the variables contained in the set $\mathbf{Z}_{\setminus 0}=\{Z_1,\dots,Z_n\}$ (for instance, $Z_0\leftarrow U_j \rightarrow Z_i\rightarrow X\rightarrow \phi(X)\rightarrow \hat{Y}$, for some $j$, $i\neq 0$; $\setminus$ is the set difference operator). These backdoor paths result in spurious correlations among the variables in the set $\mathbf{Z}$. Let $\mathbf{Z}_{cnf} = \{Z_i|Z_0\leftarrow U_j \rightarrow Z_i, j\in{1,\dots,m}, i\neq0\}$ represent the set of variables belonging to a backdoor path from $Z_0$ to $\hat{Y}$. Due to these spurious correlations, a model may rely on $\mathbf{Z}_{cnf}$ for making predictions, disregarding the importance of $Z_0$. 

\begin{definition} \textbf{(Interventional Distribution~\citep{pearl2009causality})}
    \label{interventionaldistribution} 
    The interventional distribution of a set of variables $\mathbf{Z} = \{Z_0,\dots,Z_n\}$ under an intervention to $Z_i$ with a value $z_i$, denoted by $do(Z_i=z_i)$, is defined as:
    \vspace{-3pt}
    \begin{equation}
        p(Z_1,\dots,Z_n|do(Z_i=z_i)) = \begin{cases}
    \prod_{j\neq i}p(Z_j|pa_{Z_j})& \text{if}\ \ Z_i= z_i\\
    0              & \text{if}\ \ Z_i\neq z_i
            \end{cases}
    \end{equation}
\end{definition}
The resulting probability distribution of a set of variables $\mathbf{Z}_{\setminus i}=\{Z_0,\dots,Z_n\}\setminus \{Z_i\}$ under the intervention $do(Z_i=z_i)$ is same as the probability distribution of $\mathbf{Z}_{\setminus i}$ induced by the intervened causal graph $\mathcal{G}_{do(Z_i)}$.  $\mathcal{G}_{do(Z_i)}$ is obtained by removing all incoming arrows to $Z_i$ in $\mathcal{G}$~\citep{pearl2009causality} (See Figure~\ref{fig:datagen} (c), (d)). We use $do(Z_i)$ as a shorthand for $do(Z_i=z_i)$.
\begin{definition}
\label{noconfounding}
\textbf{(No Confounding~\citep{pearl2009causality})} Given a set of variables $\mathbf{Z} = \{Z_0,\dots,Z_n\}$, an ordered pair $(Z_i,Z_j); Z_i,Z_j \in \mathbf{Z}$ is unconfounded if and only if $p(Z_i|do(Z_j))=p(Z_i|Z_j)$.
\end{definition}
\begin{definition}
\label{def:directed_info}
\textbf{(Directed Information~\citep{directed,info_theoretic})} Given a set of variables $\mathbf{Z} = \{Z_0,\dots,Z_n\}$, the directed information $I(Z_i\rightarrow Z_j)$ from $Z_i$ to $Z_j$ is defined as the conditional Kullback-Leibler divergence between the distributions $p(Z_i|Z_j), p(Z_i|do(Z_j))$ given $Z_j$. Mathematically, $I(Z_i\rightarrow Z_j)$ is defined as:
\vspace{-3pt}
\begin{equation}
 I(Z_i\rightarrow Z_j)\coloneqq D_{KL}(p(Z_i|Z_j)|| p(Z_i|do(Z_j))|p(Z_j))
 \coloneqq \mathbb{E}_{p(Z_i,Z_j)} \log \frac{p(Z_i|Z_j)}{p(Z_i|do(Z_j))}
\end{equation}
\end{definition}
\vspace{-3pt}
We now leverage directed information to define a measure of confounding in the causal model~\ref{datagen}.

\section{An Information Theoretic Measure of Confounding}
\label{sec: info theoretic}

From Definitions~\ref{noconfounding} and~\ref{def:directed_info}, the variables $Z_i$ and $Z_j$ are unconfounded if and only if $I(Z_i\rightarrow Z_j)=0$ because no confounding implies $p(Z_i|do(Z_j))=p(Z_i|Z_j)$. However, if $I(Z_i\rightarrow Z_j)>0$, it implies that $p(Z_i|do(Z_j))\neq p(Z_i|Z_j)$ and hence the presence of confounding. Also, it is important to note that the directed information is not symmetric i.e., $I(Z_i\rightarrow Z_j) \neq I(Z_j\rightarrow Z_i)$~\citep{jiao2013universal}. Since we need to quantify the notion of \textit{confounding} (as opposed to \textit{no confounding}), we leverage directed information to quantify \textit{confounding} as defined below.
\begin{definition}
\label{def:confounding}
\textbf{(An Information Theoretic Measure of Confounding)}
Given a set of variables $\mathbf{Z} = \{Z_0,\dots,Z_n\}$, the confounding $CNF(Z_i;Z_j)$ between $Z_i$ and $Z_j$ is measured as 
\begin{equation}
\label{eq: cnf}
    CNF(Z_i;Z_j) := I(Z_i\rightarrow Z_j)+I(Z_j\rightarrow Z_i)
\end{equation} 
\end{definition}
Since directed information is not symmetric, we let the confounding measure include the directed information from both directions i.e., $I(Z_i\rightarrow Z_j)$ and $I(Z_j\rightarrow Z_i)$. We now relate $CNF(Z_i;Z_j)$ with the mutual information $I(Z_i;Z_j)$ between $Z_i,Z_j$ which is later used in further analysis.
\begin{restatable}[]{proposition}{propositionone}
\label{proposition1}
In the causal graph $\mathcal{G}$ of Figure~\ref{fig:datagen} (a), we have $p(Z_i|do(Z_j))=p(Z_i)$.
\end{restatable}
\begin{proof}
In the causal graph $\mathcal{G}$ of Figure~\ref{fig:datagen} (a), let $\mathbf{U}_{cnf}=\{U_k|Z_i\leftarrow U_k \rightarrow Z_j\}$ for some $i,j$ denote the set of all confounding variables that are part of some backdoor path from $Z_i$ to $Z_j$. Then,
{
\footnotesize
\begin{equation*}
        p(Z_i|do(Z_j)) = \sum_{\mathbf{U}_{cnf}} p(Z_i|Z_j,\mathbf{U}_{cnf})p(\mathbf{U}_{cnf})
        = \sum_{\mathbf{U}_{cnf}} p(Z_i|\mathbf{U}_{cnf})p(\mathbf{U}_{cnf}) = \sum_{\mathbf{U}_{cnf}} p(Z_i,\mathbf{U}_{cnf}) = p(Z_i)
\end{equation*} 
}
The first equality is due to the adjustment formula~\citep{pearl2001direct}, and the second equality is due to the \textit{collider} structure at $X$~\citep{pearl2009causality} i.e., $Z_i\indep Z_j|\mathbf{U}_{cnf}$.
\end{proof}

\begin{restatable}[]{proposition}{propositiontwo}
\label{proposition2}
In the causal graph $\mathcal{G}$ of Figure~\ref{fig:datagen} (a), we have $CNF(Z_i;Z_i)=2\times I(Z_i;Z_j)$.
\end{restatable}
\begin{proof}
\begin{equation*}
    \begin{aligned}
    &I(Z_i\rightarrow Z_j)+I(Z_j\rightarrow Z_i)
        \stackrel{\text{Defn}~\ref{def:directed_info}}{=}\mathbb{E}_{Z_i, Z_j}\left[ \log(\frac{p(Z_i|Z_j)}{p(Z_i|do(Z_j))}) \right]+\mathbb{E}_{Z_i, Z_j} \left[ \log(\frac{p(Z_j|Z_i)}{p(Z_j|do(Z_i))}) \right]\\
        &=\mathbb{E}_{Z_i, Z_j} \left[ \log(\frac{p(Z_i|Z_j) p(Z_j|Z_i)}{p(Z_i|do(Z_j))p(Z_j|do(Z_i))}) \right]
         \stackrel{\text{Propn}~\ref{proposition1}}{=}  \mathbb{E}_{Z_i, Z_j} \left[ \log(\frac{p(Z_i|Z_j) p(Z_j|Z_i)}{p(Z_i)p(Z_j)}) \right]\\
         &=\mathbb{E}_{Z_i, Z_j} \left[ \log(\frac{p(Z_i|Z_j)p(Z_j) p(Z_j|Z_i)p(Z_i)}{p(Z_i)p(Z_j)p(Z_i)p(Z_j)})\right]
         = \mathbb{E}_{Z_i, Z_j} \left[ \log(\frac{p(Z_i,Z_j)^2}{(p(Z_i)p(Z_j))^2}) \right]\\
         &= 2\times \mathbb{E}_{Z_i, Z_j} \left[ \log(\frac{p(Z_i,Z_j)}{p(Z_i)p(Z_j)})\right]= 2 \times I(Z_i;Z_j)\\
    \end{aligned}
\end{equation*}
\end{proof}
The properties of mutual information imply that $CNF(Z_i;Z_i)$ is both non-negative and symmetric. Building upon Proposition~\ref{proposition2}, we approach the task of eliminating confounding between $Z_0$ and $Z_i$ for all $Z_i\in \mathbf{Z}_{cnf}$ as the problem of minimizing the mutual information $I(Z_0;Z_i)$ for each $Z_i\in \mathbf{Z}_{cnf}$. In the next section, we explore methodologies for minimizing $I(Z_0;Z_i)$.

\section{Removing Confounding Effects}
\label{sec: removing confounding effects}

Recall that our goal is to remove the non-causal associations from $Z_0$ to $\hat{Y}$ that go via the back-door paths, which can be achieved by minimizing $I(Z_0;Z_i);\ \forall Z_i\in \mathbf{Z}_{cnf}$ (Proposition~\ref{proposition2}). From a causal graphical model's perspective, performing interventions on $Z_0$ or $Z_i$ or both $Z_0,Z_i$ ensures $I(Z_0;Z_i)=0$ as shown in the proposition below. 

\begin{restatable}[]{proposition}{propositionthree}
\label{proposition3}
    For $\mathcal{G}_{Z_0}, \mathcal{G}_{Z_i}, \mathcal{G}_{\{Z_0,Z_i\}}$ of $\mathcal{G}$ of Figure~\ref{fig:datagen} (a), $CNF(Z_0;Z_i)=0$ for $i\neq 0$.
\end{restatable}
\begin{proof}
    For any $i\neq 0$, showing $CNF(Z_0;Z_i)=0$ is the same as showing $I(Z_0;Z_i)=0$ (Proposition~\ref{proposition2}). That is, we need to show $p(Z_0,Z_i)=p(Z_0)p(Z_i)$ (definition of mutual information). Since $X$ is a collider in each of $\mathcal{G}_{Z_0}, \mathcal{G}_{Z_i}, \mathcal{G}_{\{Z_0,Z_i\}}$ and there is no back-door path of the form $Z_0\leftarrow U_j \rightarrow Z_i$, we have $p(Z_0,Z_i)=p(Z_0)p(Z_i)$.
\end{proof}
From Proposition~\ref{proposition3}, one way of ensuring $I(Z_0;Z_i)=0;\ \forall Z_i\in \mathbf{Z}_{cnf}$ is to augment $\mathcal{D}$ with data generated from the causal models whose underlying causal graphs are: $\mathcal{G}_{Z_0}, \mathcal{G}_{\mathbf{Z}_{cnf}}, \mathcal{G}_{\mathbf{Z}_{cnf}\cup \{Z_0\}}$. That is, the augmented data should be generated from one of the following causal models~\ref{doc}-\ref{docz}.
\vspace{-3pt}
{
\small
\begin{align}
        &\mathbf{U}\sim p_{\mathbf{U}},&& Z_0 \sim p_{Z_0}, && Z_i\coloneqq f_i(pa(Z_i))\ \ i\in\{1,\dots,n\}, & X\coloneqq g(\mathbf{Z}) \label{doc}\\
        &\mathbf{U}\sim p_{\mathbf{U}},&& Z_i \sim p_{Z_i};\ \forall Z_i\in \mathbf{Z}_{cnf}, && Z_j\coloneqq f_j(pa(Z_j));\ \forall Z_j\not \in \mathbf{Z}_{cnf}, & X\coloneqq g(\mathbf{Z})\label{doz} \\
        &\mathbf{U}\sim p_{\mathbf{U}},&& Z_i \sim p_{Z_i};\ \forall Z_i\in \mathbf{Z}_{cnf}\cup \{Z_0\}, &&  Z_j\coloneqq f_j(pa(Z_j)); \ \forall Z_j\not \in \mathbf{Z}_{cnf}\cup \{Z_0\}, & X\coloneqq g(\mathbf{Z}) \label{docz}       
\end{align}
}
As explained in \S~\ref{relatedwork}, counterfactual generative networks (CGN)~\citep{cgn} generates counterfactual images by simulating causal model in Equation~\ref{docz} above, performing interventions on all of $\{Z_0\}\cup\mathbf{Z}_{cnf}$. However, performing interventions on all of $\{Z_0\}\cup\mathbf{Z}_{cnf}$ is neither necessary nor efficient. Also, in many scenarios, it is challenging to identify all possible generative factors to perform interventions. Recent methods on out-of-distribution generalization~\citep{oodgen} and invariant feature learning~\citep{modelpatching} generate counterfactuals by simulating the causal model in Equation~\ref{doz}, performing interventions on $\mathbf{Z}_{cnf}$. Traditional augmentation methods based on image manipulations such as Cutout~\citep{cutout}, CutMix~\citep{cutmix}, AugMix~\citep{augmix}, Auto Augment~\citep{autoaug}, Mixup~\citep{mixup} can be viewed as simulating causal model in Equation~\ref{dox} below, performing intervention directly on $X$. However, such models do not have causal path to $X$ from the causal feature $Z_0$ making it challenging to learn features representative of true label $Y$ when there is confounding.

\begin{equation}
    \label{dox}
    \mathbf{U}\sim p_{\mathbf{U}},\hspace{1cm} Z_i\coloneqq f_i(pa(Z_i)), \hspace{1cm} X'\coloneqq g(\mathbf{Z}), \hspace{1cm} do(X= h(X))
\end{equation}
In Equation~\ref{dox}, $h$ is a function that takes an instance $X'$ and returns a new instance $X$ after performing some changes to $X'$. The causal graphical models corresponding to models~\ref{doc},~\ref{doz},~\ref{docz}, and~\ref{dox} are shown in Figure~\ref{fig:solutions}. 
\begin{figure}
\centering
\scalebox{0.60}{
\tikzset{every picture/.style={line width=0.75pt}} 
\input{images/figure2.tikz}
}
\caption{Comparison of various interventions on $\mathcal{G}$. Few works that use these kinds of interventions are as follows. $\mathcal{G}_{do(\mathbf{Z}_{cnf})}$: ~\cite{oodgen, modelpatching} , $\mathcal{G}_{do(\{Z_0\}\cup \mathbf{Z}_{cnf})}$:~\citep{cgn,advmix}, and $\mathcal{G}_{do(X)}$: ~\citep{augmix,cutmix,cutout,mixup}. For simplicity, in this figure, assume $\mathbf{Z}_{cnf}=\mathbf{Z}_{\setminus 0}$.}
\vspace{-13pt}
\label{fig:solutions}
\end{figure}
In this paper, we propose to simulate the causal model in Equation~\ref{doc} to generate counterfactual images so that it is required to perform an intervention on only one feature $Z_0$ (Algorithm~\ref{algo:cfgeneration}). To simulate the causal models~\ref{doc}-\ref{docz}, it is necessary to identify the underlying generative factors $Z_0,\dots,Z_n$ in the presence of data exhibiting confounding bias (generated from causal model in Equation~\ref{datagen}). Once the generative factors $Z_0,\dots,Z_n$ have been identified, the process of conducting interventions and sampling images aligns with the process of counterfactual generation as formalized below.
\begin{definition}
\label{counterfactuals}
\textbf{(Counterfactual~\citep{pearl2009causality})} Given an observation $X$ with generative factors $Z_0=z_0,\dots,Z_i=z_i,\dots, Z_n=z_n$, the counterfactual $X_{cf}^i$ of $X$ w.r.t. generative factor $Z_i$ is generated using the following 3-step counterfactual inference procedure.
\vspace{-5pt}
\begin{itemize}[leftmargin=*]
\setlength\itemsep{-0.25em}
    \item \textbf{Abduction:} Recover/identify the values of $z_0,\dots,z_n$ as $z_0,\dots,z_n = g^{-1}(X)$
    \item \textbf{Action:} Perform the intervention $do(Z_i=z_i')$
    \item \textbf{Prediction:} Generate the counterfactual $X_{cf}^i$ as $X_{cf}^i=g(Z_0=z_0,\dots,Z_i=z_i',\dots,Z_n=z_n)$
\end{itemize}
\end{definition}
\begin{definition}
\label{counterfactual_identification}
\textbf{(Counterfactual Identifiability Under Confounding)} For a given observation $X$ generated using the causal model~\ref{datagen}, we say that the counterfactual $X_{cf}^i$ of $X$ is identifiable by an invertible function $\tilde{g}$ if and only if there exists an invertible function $h$ such that $z_1,\dots,z_i,\dots,z_n =h(\tilde{g}^{-1}(X))$ and $X_{cf}^i = \tilde{g}(h^{-1}(z_1,\dots,z_i',\dots,z_n));\ \forall z_i\sim p_{Z_i}$.
\end{definition}

Definition~\ref{counterfactual_identification} essentially says that if there exists an invertible function $\tilde{g}$ that identifies the underlying generative factors up to a transformation $h$, then the counterfactual $X_{cf}^i$ is identifiable i.e., Figure~\ref{commutative} commutes. Invertibility of $h$ is essential to guarantee one-to-one mapping between learned and true generative factors under confounding.

\begin{wrapfigure}[9]{r}{0.60\textwidth}
\centering
\vspace{-5pt}
\[\begin{tikzcd}
	{\mathbf{Z}} & {\tilde{\mathbf{Z}}} & X & {\mathbf{Z}} \\
	{\mathbf{Z}'} & {\tilde{\mathbf{Z}'}} & {X^i_{cf}} & {\mathbf{Z}'}
	\arrow["{g^{-1}}", maps to, from=1-3, to=1-4]
	\arrow["{do(Z_i=z_i')}", maps to, from=1-4, to=2-4]
	\arrow["g", maps to, from=2-4, to=2-3]
	\arrow["{\tilde{g}^{-1}}"', maps to, from=1-3, to=1-2]
	\arrow["h"', maps to, from=1-2, to=1-1]
	\arrow["{do(Z_i=z_i')}"', maps to, from=1-1, to=2-1]
	\arrow["{h^{-1}}"', maps to, from=2-1, to=2-2]
	\arrow["{\tilde{g}}"', maps to, from=2-2, to=2-3]
\end{tikzcd}\]
\caption{\small Commutative diagram for counterfactual identifiability}
\label{commutative}
\end{wrapfigure}
Given only observational data $\mathcal{D}$ with confounding effects, a model trained on $\mathcal{D}$ should be able to support counterfactual identification (Definition~\ref{counterfactual_identification}). This capability enables the generation of counterfactual images and facilitates subsequent data augmentation. Consequently, in the next section, we investigate how removing confounding can enhance out-of-distribution generalization and support the learning of invariant causal features.

\section{Connections to Invariant Feature Learning and Out-Of-Distribution Generalization}
\label{sec: invariant}

\noindent \textbf{Invariant Feature Learning}: In representation learning, a common approach to learn the causal/invariant feature $Z_0$ representative of a true label $Y$ is to enforce the constraint $\hat{Y}\indep Z_i|Z_0; \ \forall Z_i\in\mathbf{Z}_{cnf}$~\citep{ganin2016domain,li2018deep,condadv,modelpatching}, i.e., for a given causal feature $Z_0$, the prediction $\hat{Y}$ is independent of $Z_i;\ \forall Z_i\in \mathbf{Z}_{cnf}$. In our setting, we prove that the invariance condition $\hat{Y}\indep Z_i|Z_0;\ \forall Z_i\in \mathbf{Z}_{cnf}$ can be viewed as minimizing the confounding effects $CNF(Z_0;Z_i);\ \forall Z_i\in \mathbf{Z}_{cnf}$ along with the constraint that the prediction $\hat{Y}$ is independent of $Z_i;\ \forall Z_i\in\mathbf{Z}_{cnf}$ given $Z_0$. Concretely, consider the following expansion of $I(Z_i;\hat{Y}|Z_0)$, whose minimization is a way of enforcing $\hat{Y}\indep Z_i|Z_0$.
{\small
\begin{align*}
&I(Z_i;\hat{Y}|Z_0) = I(Z_i;\hat{Y},Z_0) - I(Z_i;Z_0)=\mathbb{E}_{Z_i,Z_0,\hat{Y}}\left[\log(\frac{p(Z_i)p(\hat{Y},Z_0)}{p(Z_i,Z_0,\hat{Y})})\right] - I(Z_i;Z_0)\\    
&=\mathop{\mathbb{E}}_{Z_i,Z_0,\hat{Y}}\left[\log(\frac{\cancel{p(Z_i)}p(Z_0)p(\hat{Y}|Z_0)}{\cancel{p(Z_i)}p(Z_0|Z_i)p(\hat{Y}|Z_0,Z_i)})\right] - I(Z_i;Z_0)=\underbrace{\mathop{\mathbb{E}}_{Z_i,Z_0,\hat{Y}}\left[\log(\frac{p(Z_0)p(\hat{Y}|Z_0)}{p(Z_0|Z_i)p(\hat{Y}|Z_0,Z_i)})\right]}_{{\footnotesize \circled{1}}} - \underbrace{I(Z_0;Z_i)}_{\frac{CNF(Z_0;Z_i)}{2}}
\end{align*}
}
In the above expansion, Since $I(Z_i;\hat{Y}|Z_0)$, the term ${\footnotesize \circled{1}}$ and $I(Z_0;Z_i)$ are always non-negative, the minimum value for $I(Z_i;\hat{Y}|Z_0)$ is obtained when: (i) $I(Z_0;Z_i)=0$, (ii) $p(Z_0)=p(Z_0|Z_i)$ and (iii) $p(\hat{Y}|Z_0)=p(\hat{Y}|Z_0,Z_i)$. Enforcing $I(Z_0;Z_i)=0$ is the same as removing confounding (Proposition~\ref{proposition2}) which will in turn ensure $p(Z_0)=p(Z_0|Z_i)$. Finally, $p(\hat{Y}|Z_0)=p(\hat{Y}|Z_0,Z_i)$ is achieved when the prediction $\hat{Y}$ is independent of $Z_i$ given $Z_0$.

\noindent \textbf{Out-Of-Distribution (OOD) Generalization: } The OOD generalization problem~\citep{oodgen,arjovsky2019invariant,buhlmann2020invariance} can also be viewed as a confounding bias removal problem. To formally establish this connection, let us consider the following scenario: the true label $Y$ can be regarded as a function $M$ of the causal feature $Z_0$ associated with $X$, that is,
\begin{equation}
\label{equationy}
    Y = M(Z_0) = M(F(X))
\end{equation}
Here $F$ is a function that extracts the causal feature $Z_0$ from $X$. Given a set of distributions $\mathcal{P}(X,Y)$ on $X, Y$, the goal in OOD generalization is to find a model $h^*$ such that the following holds~\citep{oodgen} ($\mathcal{L}$ denotes a loss function):
\begin{equation}
\label{hstar}
    h^* = \argmin_{h} \sup_{p\in \mathcal{P}} \mathbb{E}_{p}[\mathcal{L}(h(X), Y)]
\end{equation}

\begin{definition}
    \textbf{Causal Invariant Transformation ~\citep{oodgen}.} A transformation $T$ is called a causal invariant transformation if $(F\circ T)(X)=F(X);\ \forall X$. 
\end{definition}
\begin{definition}
    \textbf{Causal Essential Set ~\citep{oodgen}.} A subset $\mathcal{T}$ of all possible causal invariant transformations is called a causal essential set if for all $X_i, X_j$ such that $F(X_i)=F(X_j)$, there are finite transformations $T_1(.),\dots,T_k(.)\in \mathcal{T}$ such that $(T_1\circ\dots \circ T_k)(X_i)=X_j$.
\end{definition}
Using a causal essential set of transformations $\mathcal{T}$, it has been proved that it is possible to get $h^*$ using the augmented data $\mathcal{D}_{aug}$ generated using $\mathcal{T}$~\citep{oodgen}. In our setting, we can view counterfactual generation w.r.t. $Z_i; \ i\neq 0$ as a causal invariant transformation, augmenting counterfactuals that are generated using the simulated causal model in Equation~\ref{doz} with original data $\mathcal{D}$ aids in learning $h^*$ (Equation\ref{hstar}). 

Having examined the diverse ways of generating counterfactual images, we present a simple algorithm for generating counterfactuals by simulating causal model in Equation~\ref{doc}.

\vspace{-7pt}
\subsection{Algorithm}
\label{sec: algorithm}
\vspace{-7pt}


\begin{algorithm}
\footnotesize
\caption{Counterfactual image generation using a conditional generative model $\mathcal{M}$}
\label{algo:cfgeneration}
    \KwResult{Images sampled from a conditional generative model $\mathcal{M}$ conditioned on $Z_0$.}
    \KwData{$\mathcal{D}=\{(X_i, Y_i)\}_{i=1}^N$, $\mathbf{Z}_{cnf}$, A trained model $\mathcal{M}$, $\tau$ denoting the level of confounding.}
    $\mathcal{D}'=[]$
    
    \For{each $Z_j\in \mathbf{Z}_{cnf}$}{
    \For{each $z_0\sim Z_0 \& z_j\sim Z_j$}{
    $T = \{(X,Y)\in \mathcal{D}|Z_0 = z_0 \& Z_j=z_j$\}
    \tcp*{Filter spuriously correlated images}
    \If {$|T|/|D|>\tau$}{
    $cfs = \mathcal{M}(T)$
    \tcp*{Generate counterfactuals w.r.t. $Z_0$}
    append $cfs$ to $\mathcal{D}'$
    }
    }
    }
    return $\mathcal{D}'$
\end{algorithm}

Our objective is to employ counterfactual data augmentation to mitigate the presence of confounding bias in training data. To achieve this, we utilize a simulated causal model~\ref{doc}, where an intervention is performed on the variable $Z_0$. To simulate causal model~\ref{doc}, we use various conditional generative models, including the conditional diffusion model~\citep{ho2020denoising} (see \S~\ref{sec: experiments and results}). Previous approaches, as discussed in \S~\ref{sec: removing confounding effects}, have typically simulated one of the causal models~\ref{doz}-\ref{dox} to generate counterfactuals. However, adopting the causal model~\ref{doc} offers the advantage of requiring a single intervention solely on $Z_0$ to generate counterfactual images, in contrast to the multiple interventions required by causal models~\ref{doz}-\ref{dox}. Despite its simplicity, our proposed approach helps state-of-the-art models retain their performance compared to other ways of generating counterfactual images (see Table~\ref{results}).

\vspace{-7pt}
\section{Experiments and Results}
\label{sec: experiments and results}
\vspace{-5pt}

This section presents the experimental results on synthetic (MNIST variants) and real-world (CelebA) datasets. In order to study confounding bias, we infuse confounding in the training data and leave test data unconfounded (i.e., no spurious correlations among the generative factors; please see the Appendix for more details on implementation details). We do this to study standard generalization performance using our confounding-aware augmentation method used in the training phase.
We compare data augmentations based on causal models~\ref{doc}-\ref{dox} using standard Empirical Risk Minimization (ERM), ERM trained on unconfounded data alone (ERM-UC) in the training data, i.e., a fraction of training data that doesn't contain spurious correlations, ERM with re-weighting (ERM-RW) where multiple replicas of unconfounded data are added back to training data, conditional GAN (C-GAN)~\citep{gan}, conditional VAE (C-VAE)~\citep{vae}, Conditional-$\beta$-VAE (C-$\beta$-VAE)~\citep{betavae} ($\beta=5$ for MNIST experiments and $\beta=10$ for CelebA experiments),  AugMix~\citep{augmix}, CutMix~\citep{cutmix}, invariant risk minimization (IRM)~\citep{arjovsky2019invariant}, GroupDRO~\citep{groupdro}, CycleGAN~\citep{cycleGAN}, counterfactual generative networks (CGN)~\citep{cgn}, and conditional diffusion models (C-DM)~\citep{ho2020denoising}. 
More information on the experimental setup and qualitative results are presented in Appendix\S~\ref{sec: expsetup}.

\noindent \textbf{Colored, Double-colored, Wildlife MNIST Datasets:}
Following earlier related work, we use three synthetic datasets by leveraging the MNIST dataset~\citep{mnist} as well as its colored~\citep{arjovsky2019invariant}, textured~\citep{cgn}, and morpho~\citep{morpho} variants, which control the digit thickness (see Figure~\ref{fig:mnistvariants} and Appendix \S~\ref{sec: expsetup} for sample images). 
\begin{wrapfigure}[14]{r}{0.56\textwidth}
    \centering
    \scalebox{0.093}{
    \includegraphics{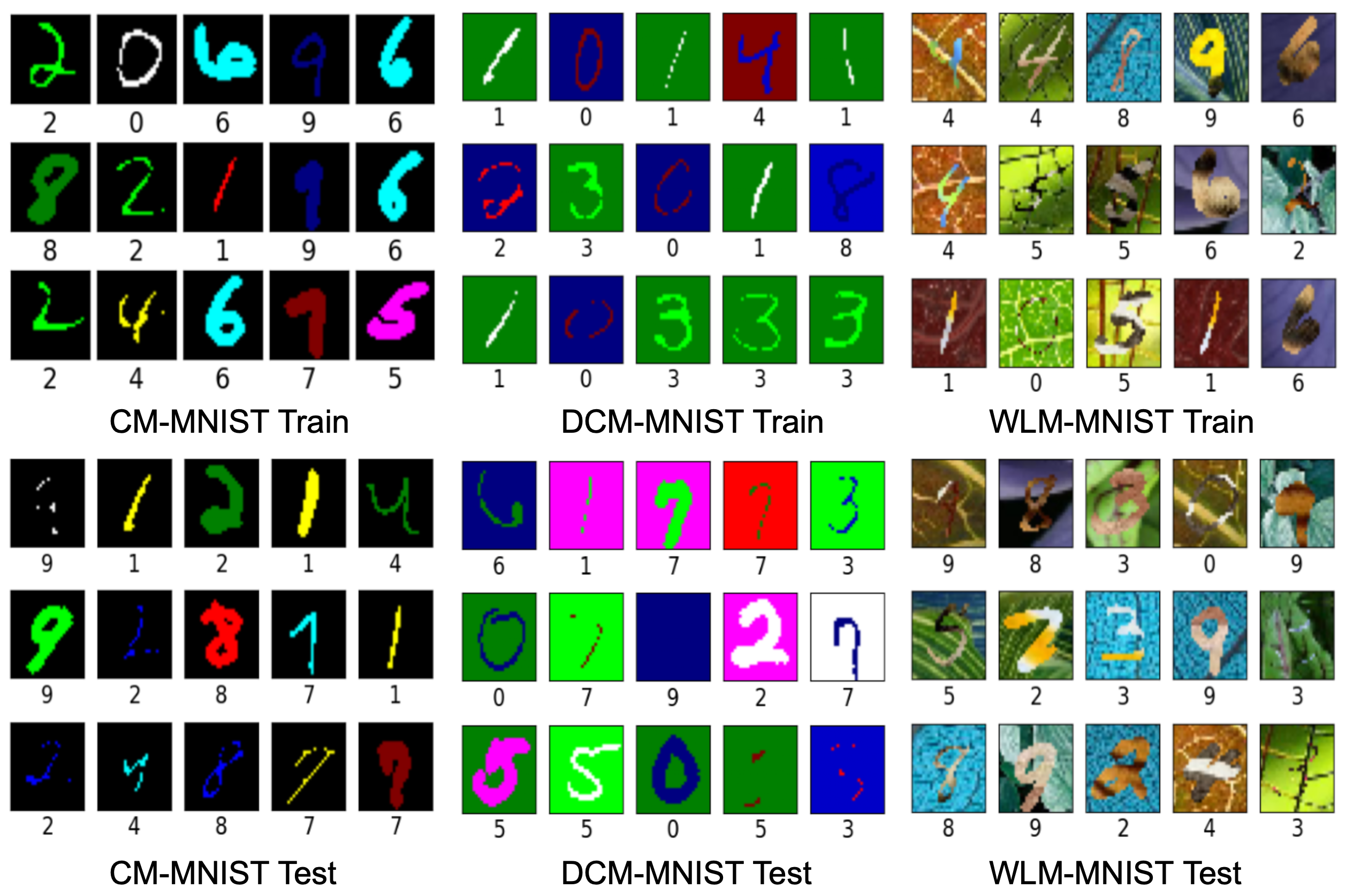}
    }
    \vspace{-7pt}
    \caption{\footnotesize Sample train and test set images of MNIST variants}
     \label{fig:mnistvariants}
\end{wrapfigure}
The three datasets are hence as follows: (i) colored morpho MNIST (CM-MNIST), (ii) double colored morpho MNIST (DCM-MNIST), and (iii) wildlife morpho MNIST (WLM-MNIST). To introduce extreme confounding among the generative factors, we implemented the following conditions. In the training set of the CM-MNIST dataset, the correlation coefficient $r$ between the digit label and digit color, denoted as $r(\text{label, color})$, is maintained as 0.95. Additionally, the digits from $0$ to $4$ are thin, while digits from $5$ to $9$ are thick. In the training set of the DCM-MNIST dataset, the digit label, digit color, and background color jointly assume a fixed set of values 95\% of the time. Specifically, we have $r(\text{label, color}) = r(\text{color, background}) = r(\text{label, background}) = 0.95$. Similar to CM-MNIST, digits from $0$ to $4$ are thin, and digits from $5$ to $9$ are thick. For the WLM-MNIST dataset's training set, the digit shape, digit texture, and background texture collectively adopt a fixed set of attribute values 95\% of the time. Furthermore, as with the previous datasets, digits from $0$ to $4$ are thin, while digits from $5$ to $9$ are thick.

In all MNIST variants discussed, the test set images exhibit no confounding bias. For instance, in the test set of DCM-MNIST, any digit can be either thin or thick, have any background color, or foreground color. Table~\ref{results} presents the results obtained from various data augmentation methods. Notably, our proposed approach, which involves performing an intervention solely on $Z_0$ to eliminate the confounding bias, helps various methods retain state-of-the-art performance compared to other counterfactual data augmentation strategies. Since conditional generative models need unconfounded data to learn conditional generation, we utilize the available unconfounded data in the training set to train all conditional generative models. As observed in Table~\ref{results}, CutMix and AugMix, both popularly used augmentation methods, demonstrate inferior performance compared to ERM-based methods. This discrepancy can be attributed to the fact that intervening on $X$ removes the causal path from $Z_0$, thereby complicating the learning of causal features (as depicted in causal model~\ref{dox} and Figure~\ref{fig:datagen} (c)). For a visual comparison of augmented images produced by different baselines, please refer to Appendix \S~\ref{sec: additional results}.

\noindent \textbf{CelebA:} Unlike MNIST variants, CelebA~\citep{liu2015faceattributes} dataset implicitly contains spurious correlations (e.g., the percentage of \textit{males} with \textit{blond hair} is different from the percentage of \textit{females} with \textit{blond hair}, in addition to the difference in the total number of \textit{males} and \textit{females} in the dataset). To further increase the confounding, we randomly subsample training data as follows: the ratio between non-blond males (60000) to blond males (20000) is $3:1$ and the ratio between non-blond females (10000) to blond females (20000) is $1:2$. In this experiment, we consider the performance of a classifier trained on the augmented data that predicts \textit{hair color} given an image. We check the performance of a downstream classifier using various data augmentation methods. Results are shown in Table~\ref{results}. The results show that the proposed counterfactual data augmentation method helps various methods retain state-of-the-art performance compared to other counterfactual data augmentation strategies. As discussed earlier, simulating causal model~\ref{doc} has the advantage that it is required to generate counterfactuals w.r.t. causal feature $Z_0$ only. Similar to the results on MNIST variants, we observe slightly lower performance for CutMix and AugMix that can be viewed as simulating causal model~\ref{dox}. Additional results on CelebA dataset are provided in Appendix \S~\ref{sec: additional results}.

\begin{table}
\centering
\caption{Test set accuracy results on MNIST variants and CelebA. Simulated interventions (Sim. Interv.) denotes the underlying interventional query used to generate counterfactuals.}
\label{results}
\footnotesize
\scalebox{0.89}{
\begin{tabular}{cl|c|c|c|c}
\toprule
\textbf{Sim. Interv.}&\textbf{Method} & \textbf{CM-MNIST} & \textbf{DCM-MNIST} & \textbf{WLM-MNIST}& \textbf{CelebA} \\
\midrule
N/A&ERM& 69.76 $\pm$ 0.21\% & 50.06 $\pm$ 0.00\%  & 41.76 $\pm$ 0.00\%& 91.21 $\pm$ 0.11\% \\
N/A&ERM-UC &64.91 $\pm$ 0.00\%&48.85 $\pm$ 0.01\%&43.98 $\pm$ 0.03\%&83.02 $\pm$ 0.50\%\\
N/A& ERM-RW &75.35 $\pm$ 1.22\%&57.40 $\pm$ 2.13\%&45.47 $\pm$ 0.87\%&92.61 $\pm$ 0.25\%\\
N/A&GroupDRO~\citep{groupdro}&61.70 $\pm$ 0.50\%&66.70 $\pm$ 0.50\%&22.20 $\pm$ 0.40\%&78.30 $\pm$ 3.10\%\\
N/A&IRM~\citep{arjovsky2019invariant}&55.25 $\pm$ 0.89\%&49.71 $\pm$ 0.71\%&50.26 $\pm$ 0.48\%&66.85 $\pm$ 4.13\%\\
\midrule
$do(X)$&AugMix~\citep{augmix}&73.04 $\pm$ 0.51\%&54.11 $\pm$ 0.12\% &36.58 $\pm$ 1.61\%&91.12 $\pm$ 0.21\%\\
$do(X)$&CutMix~\citep{cutmix}&43.68 $\pm$ 0.42\%&31.97 $\pm$ 1.67\%&16.59 $\pm$ 2.32\%&91.14 $\pm$ 0.18\%\\
\midrule
$do(Z_0\cup \mathbf{Z}_{cnf})$&CGN~\citep{cgn} &   42.15 $\pm$ 3.89\%&47.50 $\pm$ 2.18\% & 43.84 $\pm$ 0.25\%&72.86 $\pm$ 1.59\%\\
\midrule
$do(\mathbf{Z}_{cnf})$&CycleGAN~\citep{cycleGAN}&68.81 $\pm$ 1.11\%& 46.27 $\pm$ 2.14\%&34.67 $\pm$ 0.87\%&90.52 $\pm$ 1.22\%\\
\midrule
$do(Z_0)$ (Ours)&C-VAE~\citep{vae}&69.33 $\pm$ 1.20\%&51.58 $\pm$ 2.36\% &31.88 $\pm$ 1.87\%&91.33 $\pm$ 0.69\%\\
$do(Z_0)$ (Ours)&C-$\beta$-VAE~\citep{betavae} & 70.27 $\pm$ 0.50\% & 52.25 $\pm$ 1.42\% & 32.19 $\pm$ 1.58\%&91.24 $\pm$ 1.53\%\\
$do(Z_0)$ (Ours)&C-GAN~\citep{gan} & 61.30 $\pm$ 1.37\%& 40.99 $\pm$ 0.30\%& 17.50 $\pm$ 0.85\% &90.76 $\pm$ 2.77\%\\
$do(Z_0)$ (Ours)&C-DM~\citep{ho2020denoising} & \textbf{80.34$\pm$ 0.01 \%}&\textbf{73.79 $\pm$ 0.20\%}&\textbf{62.72 $\pm$ 0.02\%}&\textbf{94.73 $\pm$ 1.48\%}\\
\bottomrule
\end{tabular}
}
\vspace{-10pt}
\end{table}

\vspace{-7pt}
\section{Conclusions}
\vspace{-6pt}
In this paper, we carefully examined the detrimental impacts of confounding when performing data augmentation in DNN models. We established an association between confounding and mutual information within the considered causal processes and conducted a formal investigation of various methods for counterfactual data augmentation. Additionally, we demonstrated a strong connection between the removal of confounding and invariant causal feature learning techniques. By proposing a simple yet highly effective counterfactual data augmentation method, we showed possible methods to address the issue of confounding bias in training data. Notably, our method offers a practical solution for practitioners seeking to leverage counterfactual data augmentation to learn causal invariant features from confounded data. Our work does not present any detrimental effects on the broader scientific community.

\bibliography{bibliography}

\clearpage
\appendix
\setcounter{table}{0}
\renewcommand{\thetable}{A\arabic{table}}

\setcounter{figure}{0}
\renewcommand{\thefigure}{A\arabic{figure}}
\section*{Appendix}
\vspace{-5pt}

In this appendix, we include the following details that we could not fit into the main paper due to space constraints.
\begin{itemize}[leftmargin=*]
    \item Causality preliminaries are presented in \S~\ref{sec: preliminaries}
    \item Empirical connection between confounding and spurious correlations is presented in \S~\ref{sec: confvscorr}
    \item Experimental setup and implementation details are discussed in \S~\ref{sec: expsetup}
    \item Additional results and qualitative results are provided in \S~\ref{sec: additional results}
\end{itemize}

\vspace{-5pt}
\section{Causality Preliminaries}
\label{sec: preliminaries}
\vspace{-5pt}

\noindent \textbf{Structural Causal Models:} A Structural Causal Model (SCM) $\mathcal{S}(\mathbf{V}, \mathbf{U}, \mathcal{F}, P_{\mathbf{U}})$ encodes cause-effect relationships among a set of random variables $\{\mathbf{V}\cup \mathbf{U}\}$ in the form of a set of structural equations $\mathcal{F}$ relating each variable $X\in \{\mathbf{V}\cup \mathbf{U}\}$ with its parents $pa_{X}\in \{\mathbf{V}\cup \mathbf{U}\}\setminus \{X\}$. That is, each variable $X\in \mathbf{V}$ can be written as  $X = f(pa_{X})$ for some $f\in \mathcal{F}$. The variables in $\mathbf{U}$ are usually referred to as exogenous variables that denote uncontrolled external factors. $P_{\mathbf{U}}$ is the probability distribution of exogenous variables. The variables in $\mathbf{V}$ are usually referred as endogenous variables. 

\noindent \textbf{Causal Graphical Models:} Starting with an SCM, one can construct a directed causal graphical model $\mathcal{G}=(\mathbf{V}\cup \mathbf{U},\mathcal{E})$ as follows. $\mathcal{G}=(\mathbf{V}\cup \mathbf{U},\mathcal{E})$ is a causal graphical model in which the set of vertices $\mathbf{V}\cup \mathbf{U}$ corresponds to the set of endogenous and exogenous variables and the set of edges $\mathcal{E}$ corresponds to the set of structural equations $\mathcal{F}$ relating each variable with its parents. Concretely, if $X = f(pa_{X})$, then $\forall Y \in pa_{X}$, there exists a directed edge from $Y$ to $X$ in $\mathcal{G}$. A \textit{path} in a causal graph is defined as a sequence of unique vertices $X_1, X_2, ..., X_n$ with an edge between each consecutive vertices $X_i$ and $X_{i+1}$ where the edge between $X_i$ and $X_{i+1}$ can be either $X_i \rightarrow X_{i+1}$ or $X_{i+1}\rightarrow X_i$. A \textit{directed path} is defined as a sequence of unique vertices $X_0, X_1,..., X_n$ with an edge between each consecutive vertices $X_i$ and $X_{i+1}$ so that the the edge between $X_i$ and $X_{i+1}$ takes from $X_i \rightarrow X_{i+1}$. $Anc(X)$ is the set of all vertices that have a directed path to $X$.

A \textit{collider} is defined w.r.t. a path as a vertex $X_i$ which has a structure of the form: $\rightarrow X_i \leftarrow$ (direction of arrows imply the direction of edges along the path). A path $p$ between $X$ and $Y$ given a set of variables $\mathbf{S}$ is said to be \textit{open}, if and only if: (i) every collider node on $p$ is in $\mathbf{S}$ or has a descendant in $\mathbf{S}$, and (ii) no other non-colliders in $p$ are in $\mathbf{S}$. If the path $p$ is not open, then $p$ is said to be \textit{blocked}. $X$ and $Y$ are \textit{$d$-separated} given $\mathbf{S}$, if and only if every path from $X$ to $Y$ is blocked by $\mathbf{S}$.

A directed path starting from a node $X$ and ending at a node $Y$ is called a \textit{causal path} from $X$ to $Y$. A path that is not a causal path is called a \textit{non-causal path}. For example, the path $X\rightarrow Z\rightarrow Y$ is a causal path from $X$ to $Y$, and the path $X\leftarrow Z\rightarrow Y$ is a non-causal path from $X$ to $Y$. 
\begin{definition}
\textbf{(The Back-door Criterion)} Given a pair of variables $(X,Y)$, a set of variables $\mathbf{S}$ satisfies the backdoor criterion relative to $(X, Y)$ if no node in $\mathbf{S}$ is a descendant of $X$ and $\mathbf{S}$ blocks every backdoor path between $X$ and $Y$. 
\end{definition}
\begin{definition}
\label{def:ace}
    \textbf{(Average Causal Effect)} The Average Causal Effect (ACE) of a variable $X$ on target variable $Y$ w.r.t. at an intervention $x$ w.r.t. a baseline treatment $x^*$ is defined as
    \begin{equation*}
    ACE_X^Y \coloneqq  \mathbb{E}[Y|do(X=x)]-\mathbb{E}[Y|do(X=x^*)]
\end{equation*}
\end{definition}
If a set $\mathbf{S}$ of variables satisy the backdoor criterion relative to the pair of variables $X,Y$, the $ACE_X^Y$ can be calculated using the adjustment formula below.
\begin{equation*}
\small
\begin{aligned}
     ACE_X^Y \coloneqq  \mathbb{E}[Y|do(X=x)]-\mathbb{E}[Y|do(X=x^*)] = \mathbb{E}_{\mathbf{s}\sim\mathbf{S}}\mathbb{E}[Y|X=x,\mathbf{S}=\mathbf{s}]-\mathbb{E}_{\mathbf{s}\sim\mathbf{S}}\mathbb{E}[Y|X=x^*,\mathbf{S}=\mathbf{s}] 
\end{aligned}
\end{equation*}
\section{Confounding vs Spurious Correlation}
\label{sec: confvscorr}
Section~\ref{sec: info theoretic} of the main paper presents a way of relating confounding $CNF(Z_i;Z_j)$ and mutual information $I(Z_i;Z_j)$ between a pair of generative factors $Z_i,Z_j$. Table~\ref{tab:confandcorr} presents an empirical study that serves as evidence that confounding is directly proportional to spurious correlation between generative factors \textit{color} and \textit{digit} in the CM-MNIST dataset. We set a spurious correlation parameter $r$ while generating data. For instance, if $r=0.9$, the color and shape of CM-MNIST data take on specific predefined values 90\% of the time. We utilize a random number generator to simulate this behavior. We then evaluate Equation~\ref{eq: cnf} in the main paper using the observed data distribution. The results show the explicit relationship between confounding and spurious correlations herein. 
\begin{table}[H]
    \centering
    \footnotesize
    \begin{tabular}{llllll}
    \toprule
       Spurious correlation ($r$) & 0.10&0.20&0.50&0.90&0.95\\
        \midrule
        $CNF (color, digit)$ &0.072&0.249&1.244&3.585&4.041\\
 \bottomrule
    \end{tabular}
    \vspace{5pt}
    \caption{\footnotesize Relationship between the correlation coefficient and confounding between color and digit in CM-MNIST dataset. Correlation is directly proportional to confounding.}
    \label{tab:confandcorr}
\end{table}
\section{Implementation Details}
\label{sec: expsetup}
\noindent \textbf{Morpho MNIST:}
In this paper, we consider two transformations of MNIST images as described in~\citep{morpho}: the \textit{thin} and \textit{thick} variants of MNIST digits (additionally, we introduce confounding factors related to foreground color and background color as described in the main text). In the construction of Morpho MNIST data, we modify the thickness of digits by a specified proportion, either thinning or thickening them. Sample images demonstrating these variations can be seen in Figure~\ref{fig:thinthick}. For the training set, digits ranging from $0$ to $4$ are transformed into thin versions with a thinness value of $0.9$, while digits from $5$ to $9$ are transformed into thick versions with a thickness value of $0.9$. In the test set, digits undergo random thinning or thickening, with the thinness or thickness value determined by $\alpha$, which follows a normal distribution with a mean of $0.9$ and a standard deviation of $0.2$ i.e., $\alpha \sim \mathcal{N}(0.9, 0.2)$.

\begin{figure}
    \centering
    \includegraphics[width=1.0\textwidth]{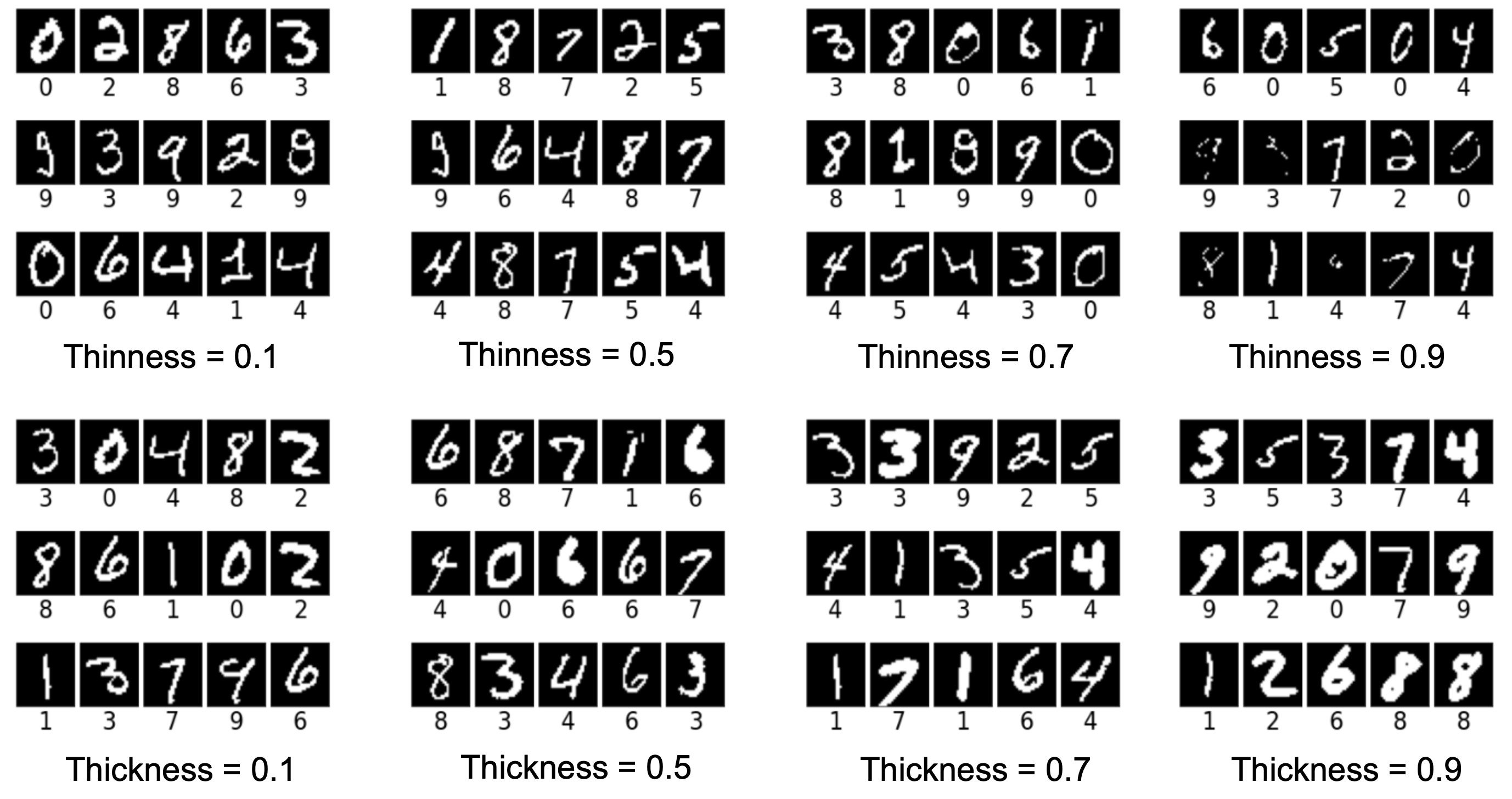}
    \caption{Morpho MNIST images for various thinness and thickness values}
    \label{fig:thinthick}
\end{figure}

\noindent \textbf{Downstream classifiers and baselines:} After performing counterfactual data augmentation, we use the following convolutional neural network (CNN) architectures to quantitatively study the usefulness of such data in various methods.

For MNIST experiments, the downstream classifier is a convolutional neural network of four convolutional layers with max-pooling after the first layer and average pooling after the fourth layer. A feed-forward layer is added at the end of the average pooling layer to make predictions. We use \textit{ReLU} activation for the internal/hidden layers and \textit{softmax} activation after the final prediction layer. For CelebA experiments, the downstream classifier is a convolutional neural network of six convolutional blocks followed by a classification/feedforward layer. Each convolutional block consists of a \textit{batch norm} layer, a convolutional layer and dropout with a probability of 0.2. We use \textit{leaky ReLU} activation for the convolutional layers and \textit{sigmoid} after the final prediction layer. We use the \textit{Adam} optimizer in all experiments.

The downstream classifiers are trained for 30 epochs in all the experiments. For each of the baselines, we use code from their official repositories. For ERM-RW, we replicate unconfounded data present in the training set multiple times such that the size of the replicated data is the same as the original dataset size. We set the number of data points to augment as a hyperparameter $\alpha$. To avoid a large search space of $\alpha$, we let $\alpha$ take on values from the set $\{1000,2000,5000,10000,20000,50000\}$. In many cases, large $\alpha$ values tend to give better results. Small $\alpha$ values are preferred when the performance saturates after a particular value of $\alpha$.

\section{Additional Results and Qualitative Results}
\label{sec: additional results}

Similar to the experiments in the main paper on CelebA, we perform an additional set of experiments by considering a different confounding setting. In this case, we consider spurious correlations between the attributes \textit{gender} and \textit{smiling}, while studying the performance of a classifier trained on the augmented data that predicts whether a person is \textit{smiling} given an image.
Concretely, we subsample the CelebA dataset such that the training set contains 37000 not-smiling males, 3000 smiling males, 10000 not-smiling females, and 40000 smiling females.

The test set contains 3000 not-smiling males, 20000 smiling males, 20000 not-smiling females, and 2000 smiling females. Similar to the results in the main paper, we see that we achieve state-of-the-art performance using counterfactual data augmentation by simulating causal model~\ref{doc}. As discussed in the main paper, simulating causal model in Equation~\ref{doc} has the advantage that it is required to generate counterfactuals w.r.t. causal feature $Z_0$ only. Since there are more images in ERM UC (at least 3000 images from each of smiling males, not smiling males, smiling females, not smiling females from the setting), we observe good results in ERM-UC. We could, however, match the performance of ERM-UC using C-DM.

\begin{wraptable}[16]{r}{0.6\linewidth}
\centering
\caption{Test set accuracy results in CelebA. Simulated interventions (Sim. Interv.) denotes the underlying interventional query used to generate counterfactuals.}
\label{appendix results}
\footnotesize
\scalebox{0.8}{
\begin{tabular}{cl|c}
\toprule
\textbf{Sim. Interv.}&\textbf{Method} &\textbf{CelebA} \\
\midrule
N/A&ERM&  80.94 $\pm$ 0.97\% \\
N/A&ERM-UC &  88.49 $\pm$ 0.13\%\\
N/A& ERM-RW & 83.12 $\pm$ 0.82\%\\
N/A&GroupDRO~\citep{groupdro}& 77.10 $\pm$ 0.30\%\\
N/A&IRM~\citep{arjovsky2019invariant}& 68.18 $\pm$ 0.24\%\\
\midrule
$do(X)$&AugMix~\citep{augmix}& 80.26 $\pm$ 0.64\%\\
$do(X)$&CutMix~\citep{cutmix}& 79.29 $\pm$ 0.69\%\\
\midrule
$do(Z_0\cup \mathbf{Z}_{cnf})$&CGN~\citep{cgn} &  74.52 $\pm$ 1.72\%\\
\midrule
$do(\mathbf{Z}_{cnf})$&CycleGAN~\citep{cycleGAN}& 82.35 $\pm$ 1.09\%\\
\midrule
$do(Z_0)$ (Ours)&C-VAE~\citep{vae}& 81.71 $\pm$ 1.83\%\\
$do(Z_0)$ (Ours)&C-$\beta$-VAE~\citep{betavae} & 80.03 $\pm$ 0.43\%\\
$do(Z_0)$ (Ours)&C-GAN~\citep{gan} & 80.13 $\pm$ 0.94\%\\
$do(Z_0)$ (Ours)&C-DM~\citep{ho2020denoising} &87.36 $\pm$ 1.20\%\\
\bottomrule
\end{tabular}
}
\end{wraptable}
The following images show the counterfactual images generated by various methods on Morpho MNIST datasets. We show counterfactual images by AugMix, CutMix that simulate causal model~\ref{dox}, CGN simulating causal model~\ref{docz}, CycleGAN simulating causal model~\ref{doz}, and conditional diffusion model~\ref{doc}. As discussed in the main paper, AugMix and CutMix, which can be seen implementing causal model~\ref{dox} cannot remove the implicit confounding in the data i.e., digit color and shape are still spuriously correlated in the augmented images. When the digits are very thin, CGN fails to capture the shape of the digit. CycleGAN and conditional diffusion models can generate good counterfactuals helping a downstream classifier to achieve good performance. 

\begin{figure}[H]
\centering
\subfigure[CM-MNIST Samples]{
 \includegraphics[width=0.32\textwidth]{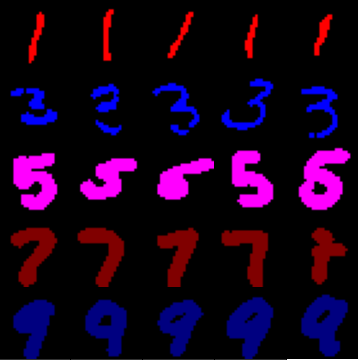}}
\subfigure[DCM-MNIST Samples]{
\includegraphics[width=0.32\textwidth]{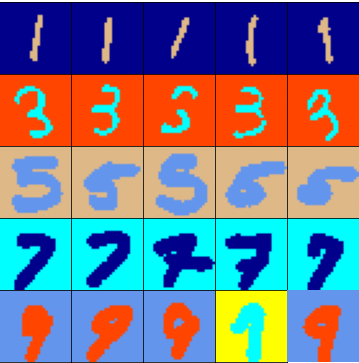}}
\subfigure[WLM-MNIST Samples]{
 \includegraphics[width=0.32\textwidth]{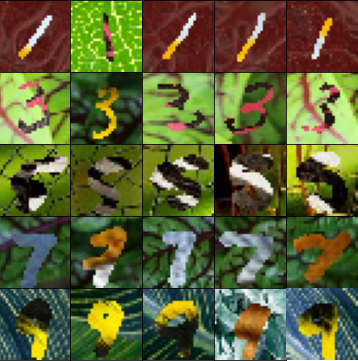}}
\end{figure}

\begin{figure}[H]
\centering
\subfigure[CM-MNIST AugMix]{
 \includegraphics[width=0.32\textwidth]{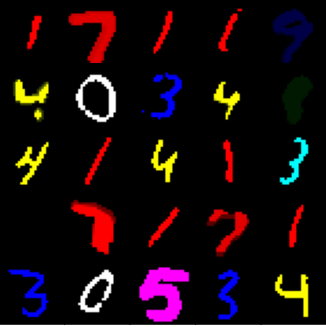}}
\subfigure[DCM-MNIST Augmix]{
\includegraphics[width=0.32\textwidth]{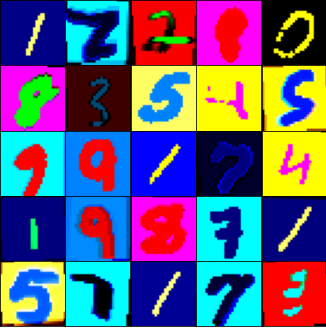}}
\subfigure[WLM-MNIST Augmix]{
 \includegraphics[width=0.32\textwidth]{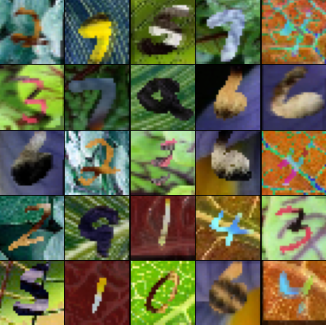}}
\end{figure}

\begin{figure}[H]
\centering
\subfigure[CM-MNIST CutMix]{
  \includegraphics[width=0.32\textwidth]{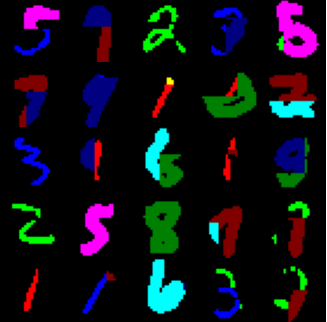}}
\subfigure[DCM-MNIST CutMix]{
  \includegraphics[width=0.32\textwidth]{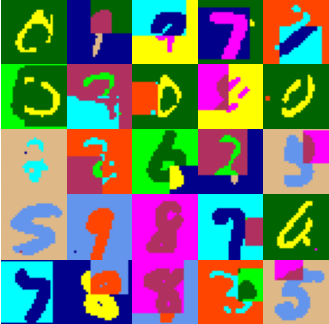}}
\subfigure[WLM-MNIST CutMix]{
  \includegraphics[width=0.32\textwidth]{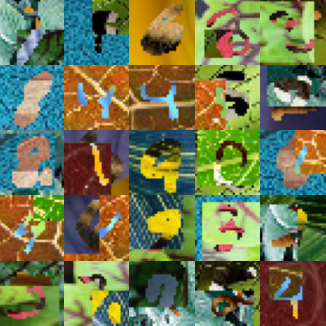}}
\end{figure}

\begin{figure}[H]
\subfigure[CM-MNIST CGN]{
  \includegraphics[width=0.32\textwidth]{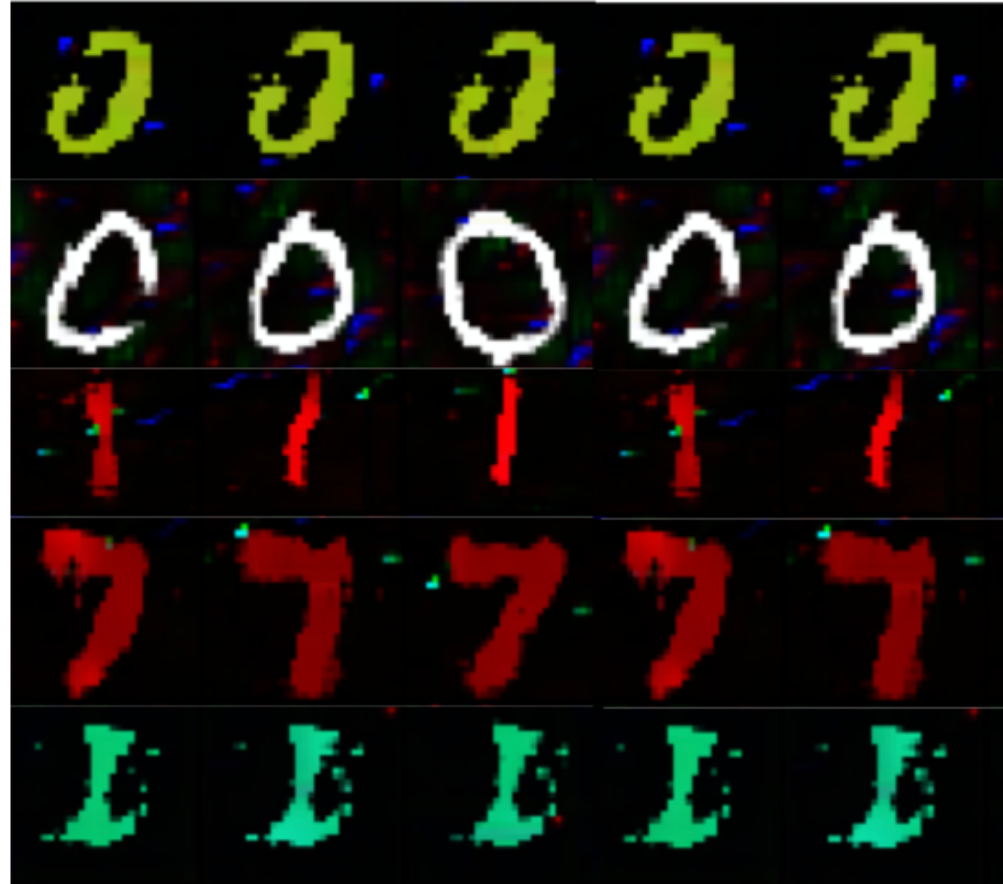}}
\subfigure[DCM-MNIST CGN]{
  \includegraphics[width=0.32\textwidth]{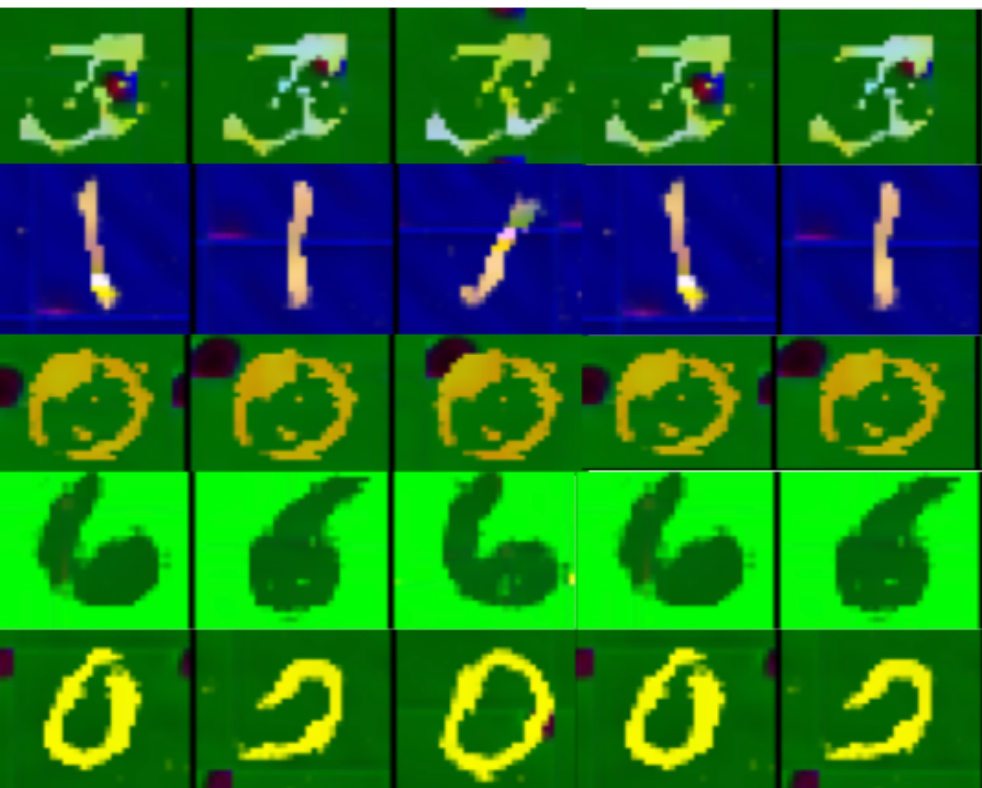}}
\subfigure[WLM-MNIST CGN]{
  \includegraphics[width=0.32\textwidth]{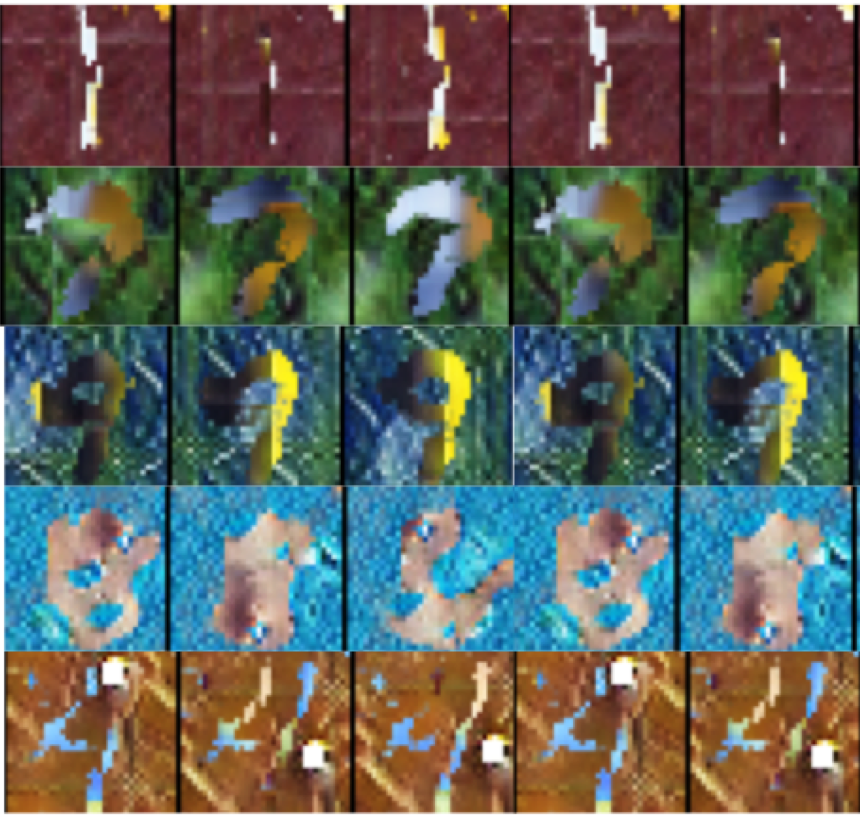}}
\end{figure}

\begin{figure}[H]
\centering
\subfigure[CM-MNIST CycleGAN]{
  \includegraphics[width=0.32\textwidth]{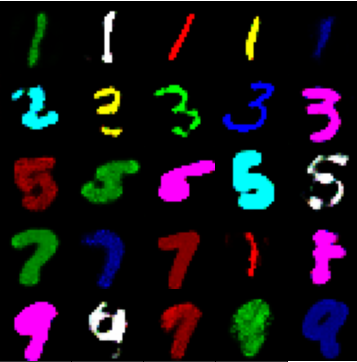}}
\subfigure[DCM-MNIST CycleGAN]{
  \includegraphics[width=0.32\textwidth]{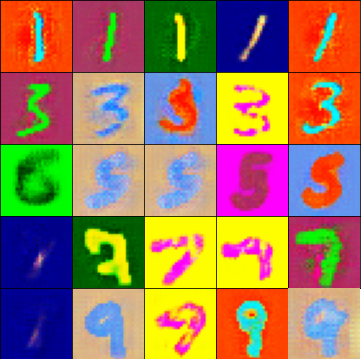}}
\subfigure[WLM-MNIST CycleGAN]{
  \includegraphics[width=0.32\textwidth]{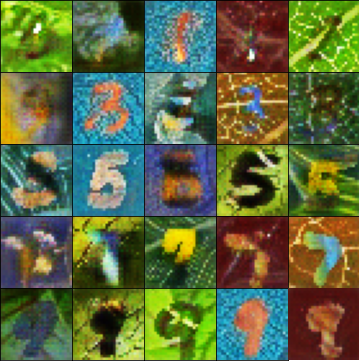}}
\end{figure}

\begin{figure}[H]
\subfigure[CM-MNIST C-DM]{
  \includegraphics[width=0.32\textwidth]{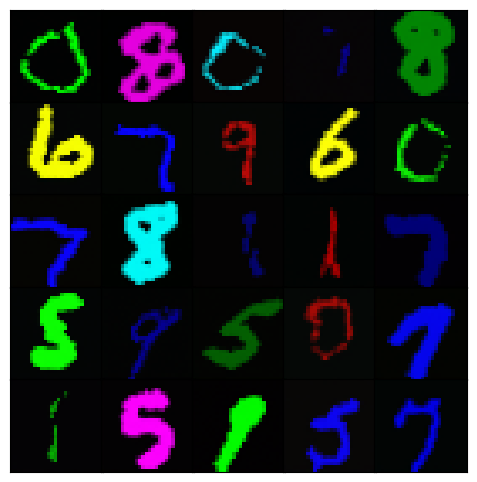}}
\subfigure[DCM-MNIST C-DM]{
  \includegraphics[width=0.32\textwidth]{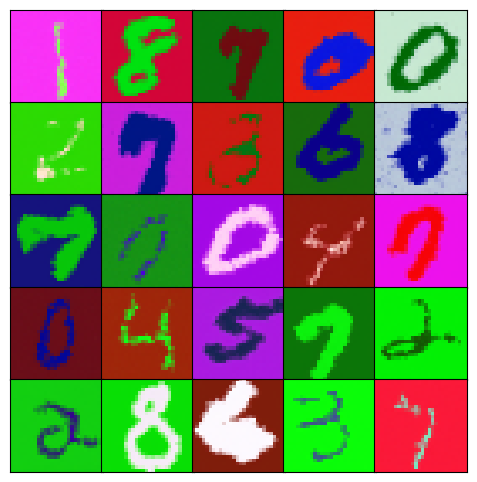}}
\subfigure[WLM-MNIST C-DM]{
  \includegraphics[width=0.32\textwidth]{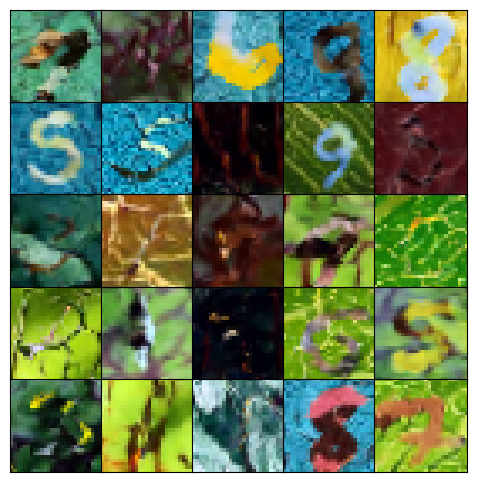}}
\end{figure}

\end{document}